%% file: arxiv.tex
\documentclass[11pt]{article}

\usepackage[utf8]{inputenc} % allow utf-8 input
\usepackage[T1]{fontenc}    % use 8-bit T1 fonts
\usepackage{hyperref}       % hyperlinks
\usepackage{url}            % simple URL typesetting
\usepackage{booktabs}       % professional-quality tables
\usepackage{amsfonts}       % blackboard math symbols
\usepackage{nicefrac}       % compact symbols for 1/2, etc.
\usepackage{microtype}      % microtypography
\usepackage{xcolor}         % colors
\usepackage[margin=1in]{geometry}
\usepackage{amsmath, amssymb, amsthm, verbatim, algpseudocode, algorithm, graphicx, natbib}
\input{macros.tex}
\allowdisplaybreaks
\title{Label differential privacy via clustering}

\author{
	Hossein Esfandiari, Vahab Mirrokni, Umar Syed, Sergei Vassilvitskii \\
	Google Research\\
	\texttt{\{esfandiari,mirrokni,usyed,sergeiv\}@google.com}}
\date{}

\begin{document}

\maketitle

\begin{abstract}
We present new mechanisms for \emph{label differential privacy}, a relaxation of differentially private machine learning that only protects the privacy of the labels in the training set. Our mechanisms cluster the examples in the training set using their (non-private) feature vectors, randomly re-sample each label from examples in the same cluster, and output a training set with noisy labels as well as a modified version of the true loss function. We prove that when the clusters are both large and high-quality, the model that minimizes the modified loss on the noisy training set converges to small excess risk at a rate that is comparable to the rate for non-private learning. We describe both a centralized mechanism in which the entire training set is stored by a trusted curator, and a distributed mechanism where each user stores a single labeled example and replaces her label with the label of a randomly selected user from the same cluster. We also describe a learning problem in which large clusters are necessary to achieve both strong privacy and either good precision or good recall. Our experiments show that randomizing the labels within each cluster significantly improves the privacy vs. accuracy trade-off compared to applying uniform randomized response to the labels, and also compared to learning a model via DP-SGD.

% The privacy of the distributed mechanism increases with the number of users per cluster, a form of \emph{privacy amplification}. 

\end{abstract}

\section{Introduction}

The goal of \emph{differentially private machine learning} is to train predictive models while preserving the privacy of user data in a training set. Most differentially private learning algorithms protect the privacy of every feature of every training example, and consequently inject so much noise into the learning process that they significantly underperform their non-private counterparts with respect to the utility of the learned model (see for instance the results on CIFAR-10 for DP-SGD \citep{abadi2016deep}). Differentially private learning algorithms also typically need full access to the private training data. These constraints can be a poor fit for many applications.

For example, consider a hospital that wants to use demographic data to train a diagnostic model for a rare illness. The input features to the model (such as a patient's age, sex, and race) may be far less sensitive than the label (whether the patient has the disease). Also, building an accurate predictive model is a hands-on, trial-and-error process that requires technically sophisticated data scientists, and the hospital is likely to achieve better results if it can share the training data with outside experts instead of having to keep all the data in-house.

\emph{Label differential privacy}, introduced by ~\cite{chaudhuri2011sample}, relaxes the goal of differentially private machine learning so that only the privacy of the training labels is protected, since in many applications that is the only sensitive user attribute. In this paper, we propose differentially private mechanisms that add noise to the labels in a training set, and then output the noisy training set and a modified loss function, where the latter corrects for the noise added by the mechanism. A learner who wants to build a predictive model can use the output of our mechanism to freely experiment with modeling choices without observing any private user data. 

%\emph{Randomized response} is a simple method for protecting the privacy of a single user attribute.

Our approach is to use a variant of \emph{randomized response} \citep{warner1965randomized} to achieve label differential privacy. We cluster training examples according to their (non-private) features, and when randomizing an example's label, we choose the replacement label from the label distribution of the example's cluster instead of from the uniform distribution.  We show that this improves the privacy vs. utility tradeoff for learning from the noisy training data when the clusters are large and the examples within a cluster have similar conditional label distributions, a property we call low \emph{cluster heterogeneity}. In particular, we show that an oracle that minimizes the modified loss function on the noisy training set outputs a model whose excess risk %(relative to the optimal model for the true data distribution) 
depends on the number of samples and the desired level of privacy, as well as the quality and size of the clusters. 

% EXPERIMENT both theoretically and empirically 

Our approach requires users to be able to privately sample from the label distribution of their example's cluster. We first describe a mechanism which uses a trusted server to perform the sampling and forwards the result to the learner. %who then sends the noisy training set to the learner, and 
%We establish the privacy of this mechanism from the learner's perspective. 
We also study a peer-to-peer setting where users are able to exchange messages with each other without a server's intervention. %, which better addresses the concern of sharing sensitive user data with a third-party. 
For this setting, we describe a distributed mechanism in which each user requests a noisy label from exactly one user in their cluster and then forwards that label to the learner. We prove that, from the learner's perspective, the privacy of this mechanism increases with the number of users per cluster.

%On the negative side, we show that taking advantage of large clusters is necessary for providing differential privacy in certain learning tasks. In order to show this, we present a basic learning task based on the clustering of the examples, and study the hardness of $\eps$-differentially private models in that context. Specifically, we show that if the clusters are smaller than a certain bound, any $\eps$-differentially private model with a constant $\eps$, either has a sub-constant precision or a sub-constant recall.

\begin{table}[t!]
    \centering
    \begin{tabular}{c|c|c}
        Algorithm & Excess risk & Comments\\
        \hline
        Optimal  & $\tilde{O}\left(\sqrt{\frac{d}{n}}\right)$ & Not private\\%[0.5cm]
        \hline
        Beimel \emph{et al} (2013) & $\tilde{O}\left(\sqrt{\frac{d}{\epsilon n}}\right)$ & Binary-labels only. Inefficient algorithm. \\%[0.5cm]
        \hline
        Bassily \emph{et al} (2018)  & $\tilde{O}\left(\frac{d^{3/5}}{(\epsilon n)^{2/5}}\right)$ & Binary labels only.\\%[0.5cm]
        \hline
        Our centralized mechanism & $\tilde{O}\left(K\sqrt{\frac{d}{n}} + \frac{K^2 \phi}{1 + (e^\epsilon - 1)\phi}\right)$ &  Cluster size $s \ge \frac{1}{\phi \epsilon}$\\%[0.5cm]
        \hline
        Our peer-to-peer mechanism & $\tilde{O}\left(\sqrt{\frac{d}{\epsilon n}} + \phi\right)$  &  Binary labels only, cluster size $s \ge \frac{1}{\phi \epsilon^2}$\\
        & & $(\epsilon, \delta)$-DP with $\delta = \frac{1}{s^2}$
    \end{tabular}
    \caption{Summary of our results. Let $n$ denote the size of the training set, $d$ the dimension of hypothesis class, $K$ the number of classes, $s$ the minimum cluster size, and $\phi$ the cluster heterogeneity.}
    \label{tab:results}
%    \vspace{-0.3in}
\end{table}

{\bf Our contributions:} We present our main results in Table \ref{tab:results}. 
\begin{itemize}\itemsep=0in
    \item In Section \ref{sec:central} we describe a centralized cluster-based randomized response mechanism with  
%that satisfies $\eps$-label differential privacy and has 
excess risk at most $\tO\left(K\sqrt{\frac{d}{n}} + \frac{K^2\phi}{1 + (e^\eps - 1)\phi}\right)$, where $n$ is the size of the training set, $d$ is the dimension of hypothesis class, $\phi$ is the cluster heterogeneity and $K$ is the number of classes. %, provided that the minimum cluster size is at least $\Omega\left(\frac{1}{\eps\phi}\right)$. 
Note that the privacy parameter $\eps$ appears in a separate term as the dimension $d$ in the excess risk bound, and thus the convergence rate of the dimension-dependent term {\em matches the optimal non-private convergence rate.} %$\tilde{O}\left(\sqrt{\frac{d}{n}}\right)$.
Also note that the dimension-free term is 
small if either $\eps$ is large \emph{or} $\phi$ is small, and so there is no cost of privacy if the clustering is good enough. %Our mechanism can be used for multiclass classification (with a corresponding scaling of the excess risk bound by the number of classes $K$) and our analysis includes, as a special case, an upper bound on the excess risk of applying uniform randomized response to the labels. 
%To the best of our knowledge, no comparable results have appeared in the literature on label differential privacy.
%
\item In Section \ref{sec:protocol} we describe a peer-to-peer cluster-based randomized response mechanism that satisfies $\left(\eps, \frac{1}{s^2}\right)$-label differential privacy and has excess risk $\tO\left(\sqrt{\frac{d}{\eps n}} + \phi\right)$ for binary classification problems, where $s$ is the minimum cluster size. While this is worse than the best-known convergence rate, % $\tO\left(\sqrt{\frac{d}{\eps n}}\right)$, 
our mechanism only involves label flipping and empirical risk minimization, and is therefore significantly more practical than existing mechanisms that run in exponential time, % than the mechanism that achieves the best-known rate, which runs the exponential mechanism on $\Omega(2^d)$ hypotheses. Our peer-to-peer mechanism 
and also does not require a trusted server.
\item In Section \ref{sec:lower} we present a hardness result relevant to multiclass classification and label differential privacy. %Even if the clustering was perfectly pure ($\phi = 0$), our excess risk bound would be $\tO\left(K\sqrt{\frac{d}{n}} + \frac{K}{s\eps}\right)$, and 
Our hardness result suggests that a residual $\frac{K}{s\eps}$ term cannot be avoided even when the clustering is pure. To prove the hardness result we develop a probabilistic analysis method that bounds the performance of any differential privacy mechanism.
\item Finally, in Section \ref{sec:experiments} we present experiments showing that our mechanisms can leverage a good clustering to improve the privacy vs. utility trade-off, outperforming both uniform randomized response and DP-SGD on real data.
\end{itemize}

\section{Related work}
There is an extensive literature on differentially private machine learning. The most common techniques include output and objective perturbation \citep{chaudhuri2011differentially} and gradient perturbation \citep{abadi2016deep}. In comparison, label differential privacy has received much less attention. \citet{chaudhuri2011sample} introduced the concept and proved a lower bound on excess risk. \citet{beimel2013private} proved an upper bound for an inefficient mechanism, while \citet{bassily2018model} described the first efficient mechanism with a non-trivial excess risk bound. Their work is the most closely related to ours, since they use a PAC oracle as a black box to learn a model on a private training set. Most previous work relies on a trusted server to implement the differentially private mechanism, with the notable exception of \citep{wang2019sparse}, who studied sparse linear regression in the local model. We will say more about the connections between previous work and our contributions when presenting our results below.

Our work is also connected to several areas of research in non-private machine learning, including learning from label proportions \citep{quadrianto2009estimating} and learning from noisy labels \citep{natarajan2013learning}.

\section{Preliminaries}
\label{sec:preliminaries}

Let $\CX$ be the \emph{example space}. Let $\CY$ be the \emph{label space}, with $K = |\CY| < \infty$. Let $D \in (\CX \times \CY)^n$ denote a \emph{dataset} of $n$ labeled examples, with $(x_i, y_i)$ denoting the $i$th element of $D$. 

For each $x \in \CX$ let $c_x \in \CC$ be the \emph{cluster} of example $x$, where $\CC$ is the set of all clusters and $C = |\CC| < \infty$. In our setting, clusters are determined using unlabeled data (\emph{i.e.}, unsupervised clustering), and since unlabeled data is not private and typically very abundant, all of our theoretical analysis will assume that the cluster $c_x$ of each example $x$ is given. Let $n_c(D) = |\{(x_i, y_i) \in D : c_{x_i} = c\}|$ be the size of cluster $c$ in dataset $D$.

% EXPERIMENT In our experiments

%\subsection{Distributions}

Let $\CP$ be a distribution on $\CX \times \CY$. Let $(X, Y) \sim \CP$ denote that $(X, Y)$ is drawn from $\CP$ and let $X \sim \CP_{\CX}$ denote that $X$ is drawn from the marginal distribution of $\CP$ on $\CX$. We write $D \sim \CP^n$ to indicate that dataset $D$ contains $n$ labeled examples each drawn independently from $\CP$. Let
\begin{align*}
p(y | x) &= \textPr{(X, Y) \sim \CP}[Y = y ~|~ X = x]\\
\hp_{y|x}(D) &= \frac{|\{(x_i, y_i) \in D : x_i = x \wedge y_i = y\}|}{|\{(x_i, y_i) \in D : x_i = x\}|}
\end{align*}
denote true conditional probability and empirical conditional probability, respectively, of label $y \in \CY$ for example $x \in \CX$. With a slight abuse of notation, let
\begin{align*}
p(y | c) &= \textPr{(X, Y) \sim \CP}[Y = y ~|~ c_X = c]\\
\hp_{y|c}(D) &= \frac{|\{(x_i, y_i) \in D : c_{x_i} = c \wedge y_i = y\}|}{|\{(x_i, y_i) \in D : c_{x_i} = c\}|}
\end{align*}
denote true conditional probability and empirical conditional probability, respectively, of label $y \in \CY$ in cluster $c \in \CC$.

We write $\bq$ to denote arbitrary \emph{cluster label distributions}, where $q(y | c)$ is the conditional probability of label $y \in \CY$ in cluster $c \in \CC$ according to $\bq$. %We write $\bq = \bhp(D)$ to mean $q(y | c) = \hp_{y|c}(D)$ for all $y \in \CY$ and $c \in \CC$.

%\subsection{Privacy}

A pair of datasets $D, D' \in (\CX \times \CY)^n$ are \emph{label neighbors} if they contain exactly the same labeled examples except that one example's label may differ between $D$ and $D'$. A \emph{mechanism} $M : (\CX \times \CY)^n \mapsto \CO$ is a randomized algorithm that takes as input a dataset and outputs into some set $\CO$. Mechanism $M$ satisifies \emph{$(\eps, \delta)$-label differential privacy} if for all datasets $D, D'$ that are label neighbors and all subsets $O \subseteq \CO$ we have
$
\Pr[M(D) \in O] \le e^\eps \Pr[M(D') \in O] + \delta,
$
where the probability is with respect the internal randomization of $M$. Let $\eps$-label differential privacy be an abbreviation for $(\eps, 0)$-label differential privacy.

%\subsection{Learning}

Let $\CH$ be a \emph{hypothesis class} containing functions with domain $\CX$. Let $\ell : \CH \times \CX \times \CY \mapsto [0, 1]$ be a \emph{loss function} that maps each hypothesis and labeled example to a non-negative loss value. Let $R(h) = E_{(X, Y) \sim \CP}[\ell(h, X, Y)]$ be the \emph{risk} of hypothesis $h \in \CH$ with respect to loss function $\ell$. We call $R(h) - \inf_{h \in \CH} R(h)$ the \emph{excess risk} of $h$.

Define $\dim(\CH, \ell)$ to be the \emph{dimension} of loss function $\ell$ and hypothesis class $\CH$: $\dim(\CH, \ell) = \frac{\log N_\alpha(\CF)}{\log (1 / \alpha)}$, where $N_\alpha(\CF)$ is the $\alpha$-covering number of the function class $\CF = \{(x, y) \mapsto \ell(h, x, y) : h \in \CH\}$. We use covering number as our definition of dimension mostly for convenience, as it applies to any real-valued loss function and simplifies comparisons to previous work. For example, it is known \citep{mohri2018foundations} that if $\ell$ is boolean-valued (say $\ell(h, x, y) = \indic{h(x) \neq y}$ is the zero-one loss) then $\dim(\CH, \ell)$ is at most the VC dimension of $\CH$ (up to a constant factor), which permits a direct comparison with \citet{beimel2013private} and \citet{bassily2018model}. We could substitute another learning-theoretic notion of the complexity of a hypothesis class (such as pseudodimension) without significantly affecting our results.
%Let $d_{\CZ} : \CZ \times \CZ \mapsto \bbR$ be a metric on the prediction space $\CZ$. The loss function is \emph{$L$-Lipschitz} if there exists $L \ge 0$ such that for every label $y \in \CY$ and every pair of predictions $z, z' \in \CZ$ we have
%\[
%|\ell(z, y) - \ell(z', y)| \le L d_{\CZ}(z, z').
%\]

\section{Centralized mechanism}
\label{sec:central}

In this setting, the dataset is stored by the curator, who applies a privacy mechanism to the dataset and outputs a dataset with noisy labels, as well as a modified loss function. 

The centralized mechanism (Algorithm \ref{alg:centralcurator}) adds noise to the labels as follows: (1) Compute the empirical label distribution in each cluster. (2) Add noise drawn from $\textrm{Laplace}(\sigma / n_c(D))$ to each label probability in each cluster $c$. (3) Truncate the per-cluster label probabilities so they are each in the interval $[\tau, 1]$. (4) Renormalize the per-cluster label probabilities so that they form distributions. (5) With probability $\lambda$, replace each label with a random label drawn from the example's cluster label distribution.

The modified loss function output by the centralized mechanism reduces the bias that was introduced by adding noise to the labels. The modified loss is constructed by re-weighting the original loss using matrix inverses that essentially `undo' the randomization of the labels. Setting the bias correction parameter $\beta = \lambda$ in Algorithm \ref{alg:centralcurator} completely removes the effect of this randomization, in expectation (see Corollary \ref{thm:centralizedutilitynobias}). \citet{natarajan2013learning} developed this debiasing technique for the special case of binary labels, which we generalize to $K > 2$ labels. 

In the supplement we show that the renormalization procedure in Algorithm \ref{alg:centralcurator} keeps each per-cluster label probability above threshold $\tau$, which is key to proving the following privacy guarantee. % of our mechanism (Theorem \ref{thm:centralizedprivacy}). Note that dividing by the total probability mass --- the most straightforward renormalization procedure --- would not ensure that this threshold is respected. Instead, we add and subtract mass to each probability as needed. Theorem \ref{thm:welldefined} shows the renormalization procedure works as intended, and also shows that the matrix inverses used to construct the modified loss function always exist. The proof is in the supplementary material.

\begin{algorithm}
\caption{Centralized mechanism}
\label{alg:centralcurator} 
\begin{algorithmic}[1]
\Statex {\bf Parameters:} Threshold $\tau \in \left[0, \frac{1}{K}\right]$; noise scale $\sigma \ge 0$; label resampling probability $\lambda \in [0, 1)$; bias correction parameter $\beta \in [0, 1)$.
\Statex {\bf Input:} Dataset $D = ((x_1, y_1), \ldots, (x_n, y_n))$ of labeled examples.
\Statex
\State \textit{// Add noise to cluster label distributions.}
\For{$c \in \CC$}
\State \textit{// Add noise to each empirical probability and clip.}
\For{$y \in \CY$}
\State $q(y | c) \gets \max\left\{\tau, \min\left\{1, \hp_{y | c}(D) + z_{y, c}\right\}\right\}$,
\State ~~~~where $z_{y, c} \sim  \textrm{Laplace}\left(\frac{\sigma}{n_c(D)}\right)$.
\EndFor
\Statex
\State \textit{// Renormalize distribution.}
\State $\Delta_c \gets 1 - \sum_y q(y | c)$
\For{$y \in \CY$}
\If{$\Delta_c < 0$}
\State $\xi_{y, c} \gets q(y | c) - \tau$
\Else
\State $\xi_{y, c} \gets 1 - q(y | c)$
\EndIf
\EndFor
\For{$y \in \CY$}
\State $\tq(y | c) \gets q(y | c) + \frac{\xi_{y, c}}{\sum_{y'} \xi_{y', c}} \Delta_c$
\EndFor
\EndFor
\Statex
\State \textit{// Randomize labels.}
\For{$(x_i, y_i) \in D$}
\State $\ry_i \gets y$ with probability $\tq(y | c_{x_i})$. \label{line:randomlabel}
\State $\ty_i \gets
\begin{cases}
y_i & \textrm{ with probability } 1 - \lambda\\
\ry_i &  \textrm{ with probability } \lambda
\end{cases}$
\EndFor
\Statex
\State \textit{// Construct noisy dataset.}
\State $\tD \gets ((x_1, \ty_1), \ldots, (x_n, \ty_n))$.
\Statex
\State \textit{// Construct modified loss function.}
\State For each $x \in \CX$ let $\tbQ_{x, \beta} \in \bbR^{K \times K}$ be the label randomization matrix defined by
\[
\tQ_{x, \beta}[y', y] = (1 - \beta)\indic{y' = y} + \beta \tq(y' | c_x).
\] \label{line:labelrandomizationmatrix}
\State Define the loss function $\tell : \CH \times \CX \times \CY \mapsto \bbR$ as
\[
\tell(h, x, y) = \sum_{y'} \tQ^{-1}_{x, \beta}[y', y] \ell(h, x, y').
\] \label{line:modifiedlossfunction}
\State {\bf return} Dataset $\tD$ and loss function $\tell$.
\end{algorithmic}
\end{algorithm}

%\begin{thm}[Well-definedness] \label{thm:welldefined} In Algorithm \ref{alg:centralcurator}, the cluster label distributions $\tbq$ satisfy $\tq(y | c) \in [\tau, 1]$ and $\sum_{y' \in \CY} \tq(y' | c) = 1$ for every label $y \in \CY$ and cluster $c \in \CC$. Also, each label randomization matrix $\bQ_{x, \beta}$ is invertible.\end{thm} 

%\subsection{Privacy}

%Having established that the cluster label distributions $\tbq$ assign at least $\tau$ probability to each label, establishing that the centralized mechanism satisfies label differential privacy is straightforward.

\begin{thm}[Centralized privacy] \label{thm:centralizedprivacy} The centralized mechanism (Algorithm \ref{alg:centralcurator}) satisfies $\eps$-label differential privacy with
$
\eps = \frac{1}{\sigma} + \log\left(1 + \frac{1 - \lambda}{\lambda \tau}\right).
$\end{thm}

%\subsection{Utility}

Given a noisy dataset $\tD$ and a modified loss function $\tell$ output by Algorithm \ref{alg:centralcurator}, our goal is to upper bound the excess risk (also called the \emph{generalization error}) of the hypothesis $\tih$ that minimizes the average of $\tell$ on $\tD$. A key benefit of such a guarantee is that it is agnostic to the internal operation of the learning algorithm, and thus applies to any algorithm for empirical risk minimization.

The minimum excess risk we can achieve, and the rate at which we approach that excess risk, will depend on both the size and quality of the clusters. We measure the quality of the clusters in terms of their \emph{heterogeneity}.

\begin{defn}[Cluster heterogeneity] \label{defn:heterogeneity} Let
$
\phi = \E_{X \sim \CP_{\CX}}\left[\sum_y \left \lvert p(y | X) - p(y | c_X)\right \rvert \right]
$
be the average total variation distance between the conditional label distribution of an example and its cluster.\end{defn}

If clusters have low heterogeneity then it should be easier to add privatizing noise to the labels without impacting utility because, intuitively, one can swap labels among examples in the same cluster without badly distorting the original data distribution. Our analysis confirms this intuition.

%The centralized mechanism (Algorithm \ref{alg:centralcurator}) constructs cluster label distributions $\tbq$ using input dataset $D$ and then resamples the labels of some labeled examples $(x, y) \in D$ according to $\tbq(\cdot | c_x)$. In Section \ref{sec:protocol} we present a peer-to-peer mechanism (Algorithm \ref{alg:peertopeeruser}) that resamples labels according to a $\tbq$ that is defined implicitly and never actually constructed by the mechanism. In both cases, a useful concept for analyzing the utility of the mechanism is the \emph{cluster distortion} of $\tbq$, which measures how well $\tbq$ matches the empirical label distributions in each cluster. This quantity does not appear explicitly in the statements of our utility guarantees, but is used in their proofs (see the proof sketches of Theorems \ref{thm:centralizedutility} and \ref{thm:peertopeerutility}).

%\begin{defn}[Cluster distortion] \label{defn:clusterdistortion} For any mechanism that takes as input a dataset $D$ and defines cluster label distributions $\tbq_D$ let
%$
%\psi = \E_{D \sim \CP^n}\left[\max_c \E\left[\sum_y \left\lvert \tq_D(y | c) - \hp_{y | c}(D) \right\rvert\right]\right]
%$
%be the expected maximum total variation between the empirical cluster label distributions and $\tbq_D$.
%\end{defn}

%We are now ready to state the main result of this section.

\begin{thm}[Centralized utility] \label{thm:centralizedutility} Let $\tD$ and $\tell$ be the dataset and loss function output by the centralized mechanism (Algorithm \ref{alg:centralcurator}) with threshold $\tau$, noise scale $\sigma$, resampling probability $\lambda$, and bias correction $\beta$, and dataset $D$ as input, and assume each cluster in $D$ has size at least $s$. Let
$
\tih = \arg \min_{h \in \CH} \sum_{(x, y) \in \tD} \tell(h, x, y)
$
be the hypothesis in $\CH$ that minimizes $\tell$ over $\tD$. Then with probability $1 - \gamma$ over the choice of $D \sim \CP^n$
\begin{align*}
E[R(\tih)] - \inf_{h \in \CH} R(h) \le   \frac{CK}{1 - \beta}\sqrt{\frac{\dim(\CH, \ell) \log \frac1\gamma}{n}}
 + \frac{CK|\beta - \lambda|}{1 - \beta}\left(\phi + \frac{K\sigma}{s} + K\tau\right)
\end{align*}
where $C > 0$ is a universal constant and the expectation is with respect to the Laplace random variables (the $z_{y, c}$'s) in Algorithm \ref{alg:centralcurator}.
\end{thm}

\subsection{Discussion}

Taken together, Theorems \ref{thm:centralizedprivacy} and \ref{thm:centralizedutility} specify a three-way trade-off between privacy, excess risk and convergence rate. The first term in the upper bound in Theorem \ref{thm:centralizedutility} is asymptotically zero as $n \goes \infty$ and determines the convergence rate, while the remaining terms are asymptotically non-zero when $\beta \neq \lambda$ and represent the residual excess risk when $n \goes \infty$. Thus the bias correction parameter $\beta$ of Algorithm \ref{alg:centralcurator} trades-off between excess risk and convergence rate, while the label resampling probability $\lambda$, the noise scale $\sigma$, and the threshold $\tau$ trade-off between excess risk and privacy. 

%The first of these terms is small when the clusters are homogeneous, the second term is small when the clusters are large, and the last term is small when the threshold is small. 

To illustrate these trade-offs, we consider some special cases of Theorems \ref{thm:centralizedprivacy} and \ref{thm:centralizedutility}, starting with a setting of the parameters in Algorithm \ref{alg:centralcurator} that reduces the centralized mechanism to uniform randomized response on the labels (which can of course be implemented as a local mechanism).

\begin{cor}[Uniform randomized response] \label{thm:centralizedutilitynobias} If $\eps > 0$, $\tau = \frac{1}{K}$, $\beta = \lambda = \frac{K}{K - 1 + e^\eps}$ and $\sigma = \infty$ then the centralized mechanism (Algorithm \ref{alg:centralcurator}) replaces each label with a uniform random label and satisfies $\eps$-label differential privacy. If in addition $\eps < 1$ and $D \sim \CP^n$ then with probability $1 - \gamma$ over the choice of $D$ and the randomness in the mechanism the hypothesis $\tih$ from Theorem \ref{thm:centralizedutility} satisfies
%\vspace{-0.3em}
\[
R(\tih) - \inf_{h \in \CH} R(h) =  O\left(\frac{K}{\eps}\sqrt{\frac{\dim(\CH, \ell) \log \frac1\gamma}{n}}\right)
\]
\end{cor}
%\begin{proof} If $\tau = \frac1K$ then the clipping and renormalization in Algorithm \ref{alg:centralcurator} causes each cluster label distribution $\tbq(\cdot | c)$ to be the uniform distribution regardless of the value of the Laplace random variables, so we set $\sigma = \infty$ to optimize the privacy guarantee. Also, if $\eps \le 1$ then $\frac{1}{1 - \beta} = O(\frac{K}{\eps})$\end{proof}

Despite its extreme simplicity, to the best of our knowledge the excess risk of uniform randomized response for label differential privacy has not previously been analyzed. For binary classification (\emph{i.e.}, $K = 2$) we know that $\dim(\CH, \ell) = O(d)$, where $d$ is the VC dimension of hypothesis class $\CH$, and thus the excess risk converges asymptotically to zero at a rate $\tO\left(\frac{1}{\eps}\sqrt{\frac{d}{n}}\right)$. By comparison, the convergence rate of the mechanism due to \citet{beimel2013private} is $\tO\left(\sqrt{\frac{d}{\eps n}}\right)$. However, their mechanism is significantly less practical than empirical risk minimization, as it involves running the exponential mechanism on $\Omega(2^d)$ hypotheses. \citet{bassily2018model} give an efficient algorithm that obtains a rate of $\tO\left(\frac{d^{3/5}}{(\eps n)^{2/5}}\right)$, which is a worse dependence on both $d$ and $n$. Also, previous work was limited to binary classification, while our analysis applies to multi-class classification.% (\emph{i.e.}, $K > 2$).

We now show that the convergence rate can be significantly improved when the clusters are both large and have low heterogeneity. 

\begin{cor}[Cluster-based randomized response] \label{thm:centralizedutilitybias} If $\eps > 0$, $\tau = \phi$, $\beta = 0$, $\lambda = \frac{1}{1 + (e^\eps - 1)\phi}$ and $\sigma = \frac1\eps$ then the centralized mechanism (Algorithm \ref{alg:centralcurator}) satisfies $O(\eps)$-label differential privacy. If in addition each cluster has size at least $s \ge \frac{1}{\eps\phi}$ and $D \sim \CP^n$ then with probability $1 - \gamma$ over the choice of $D$ the hypothesis $\tih$ from Theorem \ref{thm:centralizedutility} satisfies
\begin{align*}
E[R(\tih)] - \inf_{h \in \CH} R(h) =  ~ O\left(K\sqrt{\frac{\dim(\CH, \ell) \log \frac1\gamma}{n}} + \frac{K^2\phi}{1 + (e^\eps - 1)\phi}\right)
\end{align*}
where the expectation is with respect to the Laplace random variables (the $z_{y, c}$'s) in Algorithm \ref{alg:centralcurator}.
\end{cor}

If we let $K = 2$ then the dimension-dependent term in the convergence rate in Corollary \ref{thm:centralizedutilitybias} is $\tO\left(\sqrt{\frac{d}{n}}\right)$, where $d$ is the VC dimension of hypothesis class $\CH$, and this is the optimal rate for non-private learning. However, instead of converging to zero, the excess risk converges to $O\left(\frac{\phi}{1 + (e^\eps - 1)\phi}\right)$ when the minimum cluster size $s \ge \frac{1}{\eps\phi}$. Note that this residual excess risk is small when the privacy parameter $\eps$ is large \emph{or} the cluster heterogeneity $\phi$ (see Definition \ref{defn:heterogeneity}) is small. Thus there is not necessarily any cost of privacy if the clustering is good enough.

\section{Peer-to-peer mechanism}
\label{sec:protocol}

In the peer-to-peer setting, the dataset is stored in a distributed manner, with each user $i$ storing labeled example $(x_i, y_i)$. Instead of communicating with a central curator, each user sends and receives messages directly to other users. In this section we assume the labels are binary, so that each $y_i \in \{0, 1\}$.

In the peer-to-peer mechanism (Algorithm \ref{alg:peertopeeruser}), each user $i$ first adds noise to her own label, and then replaces her label with the noisy label of a user $j$. User $j$ is selected uniformly at random from among all the users in user $i$'s cluster. Note that this means we may have $i = j$, and also that user $j$ may be selected by other users besides user $i$. In other words, the mechanism is based on resampling, not permuting, the labels within a cluster.

An alternative approach would be for users to communicate with a server that randomly permutes the labels within each cluster before forwarding the data to the learner. We could analyze such a mechanism via the technique of privacy amplification by shuffling \citep{cheu2019distributed, erlingsson2019amplification}. But this approach would require a shuffling server that is trusted by all users.

%While it eliminates the need for a curator, the peer-to-peer mechanism has some drawbacks compared to the centralized mechanism in Section \ref{sec:central}. In this section we assume the labels are binary, so that each $y_i \in \{0, 1\}$. Also, while the centralized mechanism outputs a modified loss function that corrects for the bias introduced by adding noise to the labels, the peer-to-peer mechanism does not output a modified loss function, since there is no single party with knowledge of how the labels were randomized. As a result, our upper bound on excess risk (Theorem \ref{thm:peertopeerutility}) does not converge asymptotically to zero, although it does converge to small excess risk when the clusters have low heterogeneity.

%Algorithm \ref{alg:peertopeeruser} assumes that each user knows her own cluster and the set of users in her cluster. Since the cluster assignments are non-private information, the learner can simply broadcast this information to all users.

\begin{thm}[Peer-to-peer privacy] \label{thm:peertopeerprivacy} There exists a constant $C > 0$ such that if each cluster in $D$ has size at least $s$ and $\alpha = \frac{C \log s}{\sqrt{\theta s}}$ then the peer-to-peer mechanism (Algorithm \ref{alg:peertopeeruser}) satisfies $(\eps, \delta)$-label differential privacy with
\[
\eps \le O\left(\theta + \frac{\theta^{3/2}}{\sqrt{s} \log s} + \frac{\theta^{3/4}}{s^{1/4}}\right)\textrm{ and }\delta \le \frac{1}{s^2}.
\]\end{thm}

\begin{algorithm}
\caption{Peer-to-peer mechanism}
\label{alg:peertopeeruser} 
\begin{algorithmic}[1]
\Statex {\bf Parameters:} Label flipping probability $\alpha \in [0, 1]$; subsampling rate $\theta \in [0, 1]$
\Statex {\bf Assume:} Label set $\CY = \{0, 1\}$.
\Statex {\bf Input:} Dataset $D = ((x_1, y_1), \ldots, (x_n, y_n))$, where each labeled example $(x_i, y_i)$ is stored by user $i$.
\Statex
\For{user $i$}
\State \textit{// Add noise to own label.}
\State $\ry_i \gets
\begin{cases}
y_i & \textrm{ with probability } 1 - \alpha\\
1 - y_i &  \textrm{ with probability } \alpha
\end{cases}$
\Statex
\State \textit{// Select a random user in the same cluster.}
\State Select user $j$ uniformly at random from the set
$
\{j' : c_{x_{j'}} = c_{x_i}\}.
$
\Statex
\State \textit{// Replace own label with other user's noisy label.}
\State $\ty_i \gets \ry_j$
\Statex
\State \textit{// Subsample.}
\State Add $i$ to $I$ with probability $\theta$.
\EndFor
\Statex
\State \textit{// Construct noisy dataset.}
\State $\tD \gets ((x_{i_1}, \ty_{i_1}), \ldots, ((x_{i_m}, \ty_{i_m}))$, where each $i_j \in I$.
\Statex
\State {\bf return} Dataset $\tD$.
\end{algorithmic}
\end{algorithm}

%Note that if we set the label flipping probability $\alpha = 0$ then the mechanism cannot satisfy $(\eps, \delta)$-label differential privacy for any $\delta = o\left(\frac1s\right)$, because in the corner case where every user in a cluster has the same label, resampling each user's label from the labels in her cluster changes no labels, while in a neighboring dataset with one positive label in the cluster, each label will change with probability at least $\frac1s$.
While the centralized mechanism outputs a modified loss function that corrects for the bias introduced by adding noise to the labels, the peer-to-peer mechanism does not output a modified loss function, since there is no single party with knowledge of how the labels were randomized. As a result, our upper bound on excess risk (Theorem \ref{thm:peertopeerutility}) does not converge asymptotically to zero, although it does converge to small excess risk when the clusters have low heterogeneity.

\begin{thm}[Peer-to-peer utility] \label{thm:peertopeerutility} Let $\tD$ be the dataset output by the peer-to-peer mechanism (Algorithm \ref{alg:peertopeeruser}) when given dataset $D$ as input. Let
$
\tih = \arg \min_{h \in \CH} \sum_{(x, y) \in \tD} \ell(h, x, y)
$
be the hypothesis in $\CH$ that minimizes the true loss function $\ell$ over $\tD$. Then with probability $1 - \gamma$ over the choice of $D \sim \CP^n$ and the randomness in the mechanism
\begin{align*}
R(\tih) - \inf_{h \in \CH} R(h) \le O\left(\sqrt{\frac{\dim(\CH, \ell) \log \frac1\gamma}{\theta n}} + \phi + \alpha\right)
\end{align*}
\end{thm}

Combining Theorems \ref{thm:peertopeerprivacy} and \ref{thm:peertopeerutility} shows that if each cluster has minimum size $s \ge \frac{1}{\phi \epsilon^2}$ and $\eps < 1$ then the peer-to-peer mechanism satisfies $(\eps, \frac{1}{s^2})$-label differential privacy and has excess risk $\tO\left(\sqrt{\frac{d}{\eps n}} + \phi\right)$. This is worse than the $\tO\left(\sqrt{\frac{d}{\eps n}}\right)$ convergence rate obtained by \citet{beimel2013private}, but our peer-to-peer mechanism is significantly more practical, since it only consists of label flipping and empirical risk minimization, instead of requiring the exponential mechanism to be run on $\Omega(2^d)$ hypotheses. Our mechanism also does not require a central curator.

\subsection{Comparison to the shuffle model}
\label{sec:shuffle}

The shuffle model \citep{cheu2019distributed, erlingsson2019amplification} involves (at least) two servers: a curator and a shuffler. Typically, each user applies a local randomizer to her data, encrypts the noisy data using the curator's public key, and sends the encrypted data to the shuffler. The shuffler strips identifiers from the messages it receives and randomly permutes them, then forwards the messages to the curator, who decrypts them.

The major benefit of the shuffle model is that the privacy provided by the local randomizers is amplified by the shuffling procedure and increases with the number of users. However, if the curator and shuffler collude with one another, then this privacy amplification property is invalidated. In real-world implementations of the shuffle model (\emph{e.g.}, RAPPOR \citep{erlingsson2014rappor}) both servers are operated by the same entity thus limiting the privacy benefits. %large technology companies whom users are not inclined to trust. So the privacy benefits of splitting the trust among several servers versus only one are questionable.

By contrast, in our peer-to-peer model, each user receives an unencryted message from exactly one other user, and the learner need not be trusted by any user for privacy amplification to be achieved.

Of course, the peer-to-peer model has its own limitations. Unlike in the shuffle model, we have not shown a privacy amplification result that applies to any local randomizer, but only to simple label flipping. Also, each user observes the noisy label of another user, and the privacy of this label is not amplified. Indeed, it is straightforward to show that, from the perspective of each user, Algorithm \ref{alg:peertopeeruser} only satisfies $\log(\frac{1 - \alpha}{\alpha})$-label differential privacy, as well as only $(0, \frac1s)$-label differential privacy. However,  the amount of data observed by any single user is minuscule (\emph{i.e.}, a single bit).

One could implement shuffling in our peer-to-peer model by having all users in each cluster agree on a random permutation of the users, and then have each user request the noisy label of the user they are mapped to by the permutation. However, agreeing on a random permutation (say, by agreeing on a pseudorandom seed) would itself require a cryptographic protocol (such as key-agreement protocol \citep{merkle1978secure}), since the permutation must be kept secret from the learner.

Since Algorithm \ref{alg:peertopeeruser} involves a subsampling step, it is tempting to ask whether we could achieve privacy amplification in the peer-to-peer model by label flipping and subsampling alone, without exchanging messages among users. It is straightforward to see that this will not work. Since only the labels of the dataset are private, any subsampling applies to the labels only, so the learner can always construct a complete dataset in which some of the labels are replaced with $\bot$, indicating that a label was not provided by the user.  So a mechanism in which some users drop their label, but do not communicate with other users, is equivalent to randomized response on the set $\{0, 1, \bot\}$. Essentially, amplification by subsampling is only effective when users can completely remove themselves from the dataset, but this isn't possible when only the users' labels are private.

\section{Lower bound}
\label{sec:lower}
%On the negative side, we intend to show the necessity of finding \emph{large} high quality clusters.
Note that, in Theorem \ref{thm:centralizedutility}, even if we assume that we have perfect clusters (i.e. $\phi=0$) of size $s$, then by setting $\beta=0$, $\lambda=1$, $\tau=0$, and $\sigma=\frac 1 {\eps}$, we have an $\eps$-label differentially private mechanism with an excess risk of $\tO\left(\sqrt{\frac{d}{n}} + \frac{K}{s\eps}\right)$. In other words, we have the optimal non-private convergence rate plus a residual term $\frac{K}{s\eps}$, which is $\Theta(1)$ for $s=K$. In this section we motivate this relationship between the size of clusters and the number of labels.
We fix a basic learning task and show that, for any constant $\eps$ it is not possible to learn a nontrivial $\eps$-label differentially private model, when the size of high quality clusters are small. In fact, our result holds in a simpler yet relevant setting where we have access to the whole label distribution statically. This motivates the necessity of having large high quality clusters in our positive result when the number of labels is large.

We first define our learning task.
%In fact our result holds even when we have access to the whole label distribution statically. 

\begin{defn}{Label Association Problem (LAP):}\\
\textbf{Setup:} We have a dataset $D\in (\CX \times \CY)^n$, where each example $x$ appears with only one label $y$. $\CC$ is a partitioning of the data in $D$ and size of each cluster $c\in \CC$ is exactly $s$.\\
\textbf{Task:} For each cluster $c\in\CC$ we intend to learn the set of labels that are associated with examples in $c$, denoted as $\CY_c=\{y ~|~ \exists x \in \CX \text{ s.t. } c_x = c \wedge (x,y)\in D \}$. 
\end{defn}

Let $M$ be a label differentially private mechanism for LAP. $\tD=M(D)$ is a set of pairs $(c,y)$. We interpret $\tD$ as a binary classification, where the input is $(c,y)$ and the output is $1$ if $(c,y)\in \tD$. We use precision and recall defined as follows to measure the accuracy of model $\tD$. We have
\begin{align*}
\text{Precision}=  \frac{\sum_{c \in \CC}\sum_{y\in\CY_c} \E_{M}[\tD(c,y)] }
{\E_M[|\tD|]} \qquad
\text{Recall}=  \frac{\sum_{c \in \CC}\sum_{y\in\CY_c} \E_{M}[\tD(c,y)] }
{\sum_{c \in \CC}|\CY_c|}, 
\end{align*}
where the expectations are over the randomness of the mechanism $M$.
Note that without differential privacy, this problem can be learned with precision $1$ and recall $1$.

The next theorem states our main hardness result, and the proof is in the supplementary material. Our proof defines a randomized process that generates two neighboring datasets $D$ and $D'$. Then we fix an arbitrary $\eps$-label differentially private mechanism $M$ and show that if $\eps$ is a constant either recall of $M$ on $D'$ is sub-constant or precision of $M$ on $D$ is sub-constant. We do this by analysing the probability that, the label in $D'$ that is not in $D$, is preserved by mechanism $M$. If such probability is small then the the recall of $M(D')$ is small, if it is large the precision of $M(D)$ is small.
\begin{thm}\label{thm:hard}
When $s\leq o(K)$, it is impossible to have an $\eps$-differential privacy mechanism $M$ for LAP with a constant $\eps$, that guarantees a constant precision and a constant recall. 
\end{thm}

\section{Experiments}
\label{sec:experiments}

%We evaulated three mechanisms: (1) {\tt UniformRR}, which is the centralized mechanism (Algorithm \ref{alg:centralcurator}) with parameters set according to Corollary \ref{thm:centralizedutilitynobias}; (2) {\tt ClusterRR}, which is the centralized mechanism (Algorithm \ref{alg:centralcurator}) with parameters set according to Corollary \ref{thm:centralizedutilitybias}; (3)  to several training sets derived from MNIST training data \citep{}. First, we used the {\tt sklearn.cluster.KMeans} package (with default parameters) to learn 100 clusters on the unlabeled training set. Next we used the centralized mechanism () and the clusters to make the training labels private. We also used uniform randomized response (which is equivalent to the centralized mechanism with parameters taken from Corollary \ref{}) to make the training labels private. For each private training set, we used the {\tt sklearn.linear model.SGDClassifier package} (with default parameters) to learn a classifier on the modified loss function output by the mechanism, and evaluated the classifier's performance with respect to accuracy on MNIST test data. The results, which are shown in Figure \ref{}, demonstrate that cluster-based randomized response leads to much higher quality models for the same value of $\epsilon$, especially when $\epsilon$ is very small, than uniform randomized response.

We evaluated the following mechanisms on the MNIST \citep{lecun-mnisthandwrittendigit-2010}, Fashion-MNIST \citep{xiao2017fashionmnist} and CIFAR-10 \citep{Krizhevsky09learningmultiple} datasets:
\begin{itemize}
    \item {\tt UniformRR}: Algorithm 1 with parameters set according to Corollary 1.
    \item {\tt ClusterRR}: Algorithm 1 with parameters set according to Corollary 2.
    \item {\tt DP-SGD}: Differentially-private variant of SGD \citep{abadi2016deep}.
\end{itemize}
For the {\tt ClusterRR} mechanism we learned 100 clusters on each unlabeled training set using the {\tt sklearn.cluster.KMeans} package (with default parameters). For both randomized response mechanisms we used the {\tt sklearn.linear\textunderscore model.LogisticRegression} package (set to `multinomial' and using the `SAGA' solver) to learn a classifer on the noisy training set output by the mechanism. For {\tt DP-SGD} we learned a logistic regression model by adapting the implementation from the TensorFlow Privacy library \citep{TFPrivacy}. We varied the noise added to the gradients, and for each noise level computed $\epsilon$ using the privacy-by-iteration method \citep{feldman2018privacy} with $\delta = 1 / n$, where $n$ is the training set size.

For each mechanism and each dataset we evaluated the learned classifer's accuracy on the test set. See the first three panels of Figure \ref{fig:results} for results, where each data point is the average of 5 trials, and each y-axis is normalized, \emph{i.e.}, divided by the accuracy of the non-private classifier that is learned on the original training set. Observe that {\tt ClusterRR} outperforms both {\tt UniformRR} and {\tt DP-SGD} on each dataset for a wide range of the privacy parameter $\epsilon$.

We also assessed the importance of a good clustering for {\tt ClusterRR} by fixing the privacy parameter $\epsilon = 0.5$ and varying the number of clusters. See the last panel of Figure \ref{fig:results}, which shows that the performance of {\tt ClusterRR} degrades sharply when the number of clusters is very small, since in that case the clusters are quite heterogeneous.

\begin{figure}
    \centering
        \includegraphics[scale=0.18]{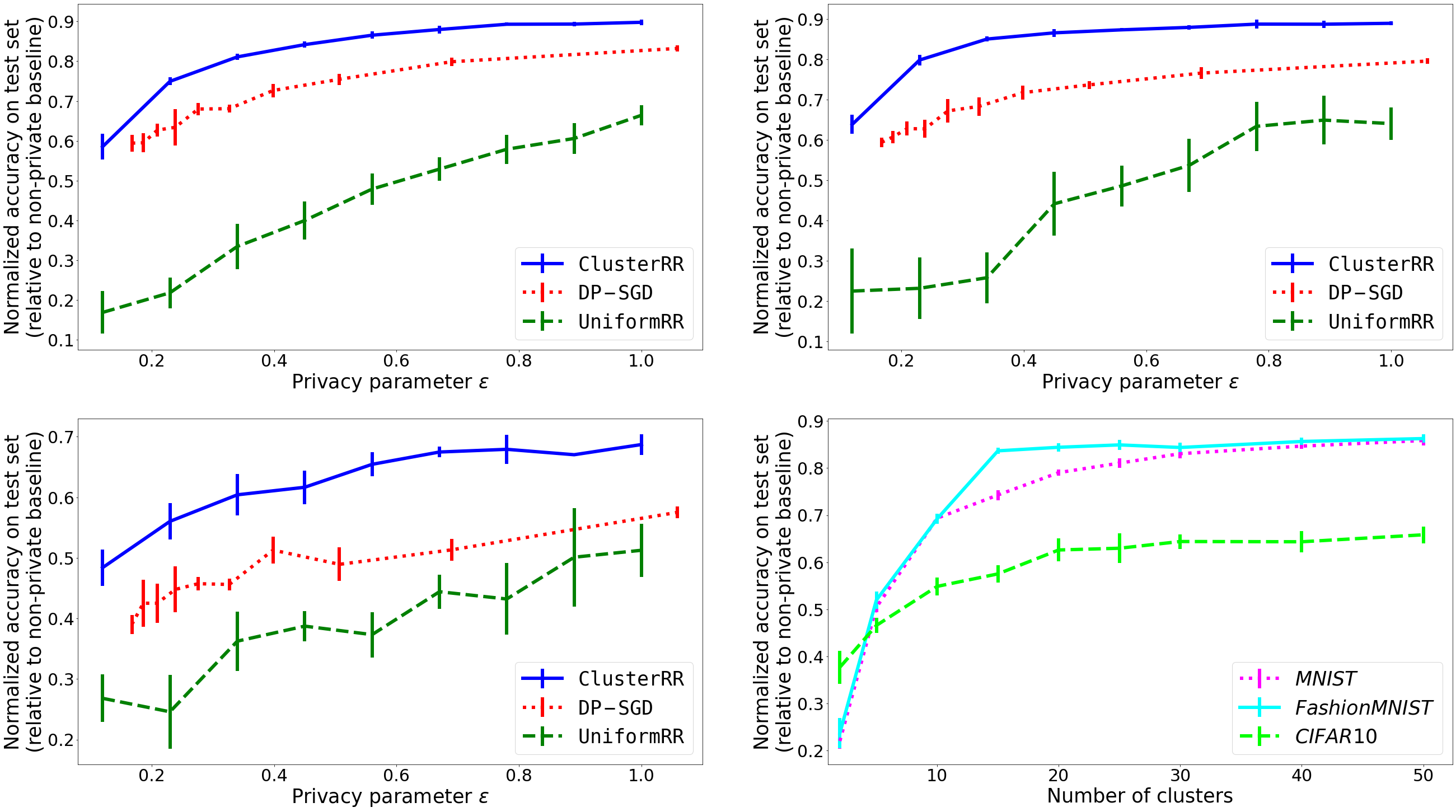}
        \caption{\label{fig:results} Performace of each mechanism on the MNIST ({\it top left}), Fashion-MNIST ({\it top right}) and CIFAR-10 ({\it bottom left}) datasets. Performance of the {\tt ClusterRR} mechanism (for $\epsilon = 0.5$) on each dataset when varying the number of clusters ({\it bottom right}). }
\end{figure}

\section{Conclusion}
In this work we presented centralized and distributed label differential privacy mechanisms. Our mechanisms are based on a clustering of examples in the training set. We upper bound the excess risk of our mechanisms by a rate comparable to that of non-private learning, especially when the clusters are both large and high-quality. We complement our results with a lower bound that illustrates why it is hard to learn privately when we do not have large high-quality clusters. We also present experimental results on real data showing that our proposed mechanisms outperform existing mechanisms for differentially private learning.

Our mechanisms are designed to protect the privacy of labels in training data, and we expect that users would feel safer in a world where more learning algorithms were privacy-preserving. However, a potential risk of the widespread adoption of our mechanisms would be their misapplication to settings where the features are also private, since that would give users a false sense of security.

\bibliography{label_dp}
\bibliographystyle{unsrtnat}

\appendix

\section{Analysis of the centralized mechanism}

\subsection{Label randomization matrices}

We first establish properties of the label randomization matrices $\tbQ_{x, \beta} \in \bbR^{K \times K}$ defined by the centralized mechanism (see line \ref{line:labelrandomizationmatrix} of Algorithm \ref{alg:centralcurator}).

\begin{lem} \label{lem:singular} The minimum singular value of any label randomization matrix $\tbQ_{x, \beta}$ is at least $\frac{1 - \beta}{\sqrt{2K}}$.\end{lem}
\begin{proof} For brevity, we drop subscripts and conditioning on $x$, letting $\bQ = \tbQ_{x, \beta}$ and $\bq = \tbq(\cdot | c_x)$ Let $\bv=<v_1,\dots,v_K>$ be an arbitrary vector such that $\|\bv\|_2=1$.  Let $\bu=<u_1,\dots,u_K>= \bQ \bv$. Note that for all $i$ we have $u_i=(1-\beta)v_{i}+\beta \bq^\top \bv$. We prove this lemma in two cases: First, all $v_i$s have the same sign. Second, there exists a $v_i$ which is negative and a $v_j$ which is positive. 

\paragraph{Case 1: All $v_i$s have the same sign.} Note that we have 
\begin{align*}
    \|\bu\|_2 &= \sqrt{\sum_{i=1}^{K} \big((1-\beta)v_{i}+\beta \bq^\top \bv\big)^2}\\
    &\geq \sqrt{\sum_{i=1}^{K} \big((1-\beta)v_{i}\big)^2} &\text{same sign}\\%&\text{since $(1-\beta)v_{i}$ and $\beta q v$ have the same sign.}\\
    &\geq \sqrt{\max_{i=1}^{K} \big((1-\beta)v_{i}\big)^2} \\
    &\geq \sqrt{\max_{i=1}^{K} \big((1-\beta)\frac 1 {\sqrt K}\big)^2} &\text{since $\sum_{i=1}^{K} v_i^2 = 1$}\\
    &= \frac{1-\beta}{\sqrt K}.
\end{align*}
\paragraph{Case 2: There exists a $v_i$ which is negative and a $v_j$ which is positive.} Note that we have 
\begin{align*}
    &\|\bu\|_2 = \\
    &\sqrt{\sum_{i=1}^{K} \big((1-\beta)v_{i}+\beta \bq^\top \bv\big)^2}\geq\\
    & \sqrt{ \big((1-\beta)\min_i v_{i}+\beta \bq^\top \bv\big)^2 + \big((1-\beta)\max_i v_{i}+\beta \bq^\top \bv\big)^2}\\
    &= (1-\beta)\sqrt{ \big(\min_i v_{i}+\frac{\beta \bq^\top \bv}{1-\beta}\big)^2 + \big(\max_i v_{i}+\frac{\beta \bq^\top \bv}{1-\beta}\big)^2}\\
    &\geq (1-\beta)\min_x \sqrt{ \big(\min_i v_{i}+x\big)^2 + \big(\max_i v_{i}+x\big)^2}=\\
    &  (1-\beta)\sqrt{ \big(\frac {\min_i v_i - \max_i v_i}{2}\big)^2 + \big(\frac {\max_i v_i - \min_i v_i}{2}\big)^2}\\
    &=  (1-\beta)\sqrt{ \big(\frac {1}{2\sqrt K}\big)^2 + \big(\frac {1}{2\sqrt K}\big)^2} \hspace{1.5cm}\\%\text{$\sum_{i=1}^{K} v_i^2 = 1$}\\
    &= \frac{1-\beta}{\sqrt {2 K}}. \qedhere
\end{align*}
\end{proof}

\begin{lem} \label{thm:absconv} Each label randomization matrix $\tbQ_{x, \beta}$ satisfies $\max_y \sum_{y'} \left|\tQ^{-1}_{x, \beta}[y', y]\right| \le \frac{\sqrt{2}K}{1 - \beta}$.\end{lem}
\begin{proof} By properties of matrix norms we have
\[
\max_y \sum_{y'} \left|\tQ^{-1}_{x, \beta}[y', y]\right| = \norm{\tbQ^{-1}_{x, \beta}}_1 \le \sqrt{K}\norm{\tbQ^{-1}_{x, \beta}}_2 \le \frac{\sqrt{2}K}{1 - \beta}
\]
where the last inequality follows from Lemma \ref{lem:singular}.\end{proof}

\subsection{Well-definedness of centralized mechanism}

\begin{thm}[Well-definedness] \label{thm:welldefined} In Algorithm \ref{alg:centralcurator}, the cluster label distributions $\tbq$ satisfy $\tq(y | c) \in [\tau, 1]$ and $\sum_{y' \in \CY} \tq(y' | c) = 1$ for every label $y \in \CY$ and cluster $c \in \CC$. Also, each label randomization matrix $\bQ_{x, \beta}$ is invertible.\end{thm} 

\begin{proof} To show that $\tq(y | c) \in [\tau, 1]$, first note that clearly $q(y | c) \in [\tau, 1]$, and therefore $\xi_{y, c} \ge 0$. So if $\Delta_c < 0$ then $\tq(y | c) \le 1$ and
\begin{align*}
\tq(y | c) &= q(y | c) + \frac{\xi_{y, c}}{\sum_{y'} \xi_{y', c}} \Delta_c\\
&= \tau + q(y | c) - \tau + (q(y | c) - \tau)\frac{\Delta_c}{\sum_{y'} (q(y' | c) - \tau)}\\
&= \tau + q(y | c) - \tau + (q(y | c) - \tau)\frac{\Delta_c}{1 - \Delta_c - K\tau}\\
&= \tau + q(y | c) - \tau + (\tau - q(y | c))\frac{-\Delta_c}{-\Delta_c + 1 - K\tau}\\
&\ge \tau + q(y | c) - \tau + \tau - q(y | c)\\
&= \tau,
\end{align*}
where we used $1 - K\tau \ge 0$. Similarly, if $\Delta_c \ge 0$ then $\tq(y | c) \ge \tau$ and
\begin{align*}
\tq(y | c) &= q(y | c) + \frac{\xi_{y, c}}{\sum_{y'} \xi_{y', c}} \Delta_c\\
&= 1 + q(y | c) - 1 + (1 - q(y | c))\frac{\Delta_c}{\sum_{y'} (1 - q(y' | c))}\\
&= 1 + q(y | c) - 1 + (1 - q(y | c))\frac{\Delta_c}{\Delta_c + K - 1}\\
&\le 1 + q(y | c) - 1 + 1 - q(y | c)\\
&= 1,
\end{align*}
where we used $K - 1 \ge 0$. Thus $\tq(y | c) \in [\tau, 1]$. Also we have $\sum_{y'} \tq(y' | c)  = 1$ because
\[
\sum_{y'} \tq(y' | c) = \sum_{y'} q(y' | c) + \sum_{y'} \frac{\xi_{y', c}}{\sum_{y''} \xi_{y'', c}} \Delta_c = 1 - \Delta_c + \Delta_c = 1.
\]
Finally, the invertibility of each label randomization matrix $\tbQ_{x, \beta}$ is immediate from Lemma \ref{lem:singular} and the fact that $\beta < 1$.\end{proof}

\subsection{Proof of Theorem \ref{thm:centralizedprivacy}}

\begin{proof} Let $M$ be the mechanism in Algorithm \ref{alg:centralcurator}. We can write $M$ as the composition of two mechanisms, $M_1$ and $M_2$, with $M(D) = M_2(D, M_1(D))$, where $M_1(D)$ outputs the noisy cluster label distributions $\tbq$, and $M_2(D, \tbq)$ uses $\tbq$ to resample the labels in $D$ to form $\tD$ and constructs the modified loss function $\tell$. By sequential composition and post-processing, if $M_1$ and $M_2$ are $\eps_1$- and $\eps_2$-differentially private, respectively, then $M$ is $(\eps_1 + \eps_2)$-differentially private.

Note that after adding $z_{y, c}$ to each $\hp_{y | c}(D)$, mechanism $M_1$ does not access dataset $D$ again. Since each $\hp_{y | c}(D)$ is computed using a disjoint subset of the dataset and has sensitivity $1 / n_c(D)$, and the scale of Laplace random variable $z_{y, c}$ is $\sigma / n_c(D)$, mechanism $M_1$ is $(1 / \sigma)$-differentially private.

Mechanism $M_2$ is just randomized response per label, using $q(\cdot | c_{x_i})$ as the random label distribution for each labeled example $(x_i, y_i)$, followed by post-processing. Thus $M_2$ is $\log(1 + (1 - \lambda) / \lambda\tau)$-differentially private, since for all $y \in \CY$ we have
\[
\frac{\Pr[\ty_i = y ~|~ y_i = y]}{\Pr[\ty_i = y ~|~ y_i \neq y]} = \frac{1 - \lambda +  \lambda \tq(y | c_{x_i}))}{\lambda \tq(y | c_{x_i})} \le 1 + \frac{1 - \lambda}{\lambda\tau}.
\]\end{proof}

\subsection{Proof of Theorem \ref{thm:centralizedutility}}

Fix threshold $\tau$ and noise scale $\sigma$. We write $(\bx, \by, \rby, \tby, \bz) \sim \CP^n_{\beta, \lambda}$ to denote the following joint distribution: Draw $(\bx, \by) = ((x_1, y_1), \ldots, (x_n, y_n)) \sim \CP^n$, run Algorithm \ref{alg:centralcurator} on input dataset $D = (\bx, \by)$ with bias correction parameter $\beta$ and label flipping probability $\lambda$, and let $\rby$, $\tby$, $\bz$ be the vectors of variables $z_{y, c}$, $\ry_i$, $\ty_i$, respectively, defined in the algorithm. Note that Algorithm \ref{alg:centralcurator} is deterministic if $(\bx, \by, \rby, \tby, \bz)$ is fixed.

Let $\hR_{\beta, \lambda}(h) = \frac1n \sum_{(x, y) \in \tD} \tell(h, x, y)$ be the empirical loss of $h$ with respect to the loss function and dataset output by Algorithm \ref{alg:centralcurator} when $(\bx, \by, \rby, \tby, \bz) \sim \CP^n_{\beta, \lambda}$, and let $R_{\beta, \lambda}(h) = \E_{\bx, \by, \rby, \tby, \bz}[\hR_{\beta, \lambda}(h)]$.

\begin{lem}[Unbiasedness of modified loss] $R_{\beta, \beta}(h) = R(h)$ for any hypothesis $h \in \CH$.\end{lem}
\begin{proof} Let $(\bx, \by, \rby, \tby, \bz) \sim \CP^n_{\beta, \beta}$. We have 
\begin{align}
R_{\beta, \beta}(h) &= \E_{\bx, \by, \rby, \tby, \bz}[\hR_{\beta, \beta}(h)] \notag\\
&= \E_{\bx, \by, \rby, \tby, \bz}\left[\frac1n \sum_i \tell(h, x_i, \ty_i)\right] \notag\\
&= \E_{\bx, \by, \rby, \tby, \bz}\left[\frac1n \sum_i \sum_{y'} \tQ^{-1}_{x_i, \beta}[y', \ty_i] \ell(h, x_i, y')\right] \notag\\
&= \E_{\bx, \by, \rby, \tby, \bz}\left[\frac1n \sum_i \sum_{y'} \sum_y \indic{\ty_i = y} \tQ^{-1}_{x_i, \beta}[y', y] \ell(h, x_i, y')\right] \notag\\
&= \E_{\bx, \by, \rby, \bz}\left[\frac1n \sum_i \sum_{y'} \sum_y \tQ_{x_i, \beta}[y, y_i] \tQ^{-1}_{x_i, \beta}[y', y] \ell(h, x_i, y')\right] \label{eq:Q}\\
&= \E_{\bx, \by, \rby, \bz}\left[\frac1n \sum_i \sum_{y'} \ell(h, x_i, y') \sum_y \tQ^{-1}_{x_i, \beta}[y', y] \tQ_{x_i, \beta}[y, y_i]\right] \notag\\
&= \E_{\bx, \by}\left[\frac1n \sum_i \ell(h, x_i, y_i)\right] \label{eq:cancel}\\
&= R(h) \notag
\end{align}
where Eq.~\eqref{eq:Q} follows from the definition of $\tbQ_{x, \beta}$ in Algorithm \ref{alg:centralcurator} (see line \ref{line:labelrandomizationmatrix}, and recall that in this case $\beta = \lambda$). We establish Eq.~\eqref{eq:cancel} by letting $\bM = \tbQ^{-1}_{x_i, \beta} \tbQ_{x_i, \beta}$ and noting
\[
M[y', y_i] = \sum_{y} \tQ^{-1}_{x_i, \beta}[y', y] \tQ_{x_i, \beta}[y, y_i]
\]
and $\bM = \bI$, and therefore $M[y', y_i] = 1$ if $y' = y_i$ and $M[y', y_i] = 0$ otherwise.\end{proof}

\begin{lem}[Boundedness of modified loss] \label{thm:boundedness} $\max_{h, x, y} \left \lvert \tell(h, x, y) \right \rvert \le \frac{\sqrt{2}K}{1 - \beta}$.\end{lem}
\begin{proof} By the definition of $\tell$ in Algorithm \ref{alg:centralcurator} (see line \ref{line:modifiedlossfunction})
\[
\max_{h, x, y} \left\lvert \tell(h, x, y) \right\rvert = \max_{h, x, y} \left\lvert \sum_{y'} \tQ^{-1}_{x, \beta}[y', y] \ell(h, x, y') \right\rvert \le \left( \max_y \sum_{y'} \left\lvert \tQ^{-1}_{x, \beta}[y', y] \right\rvert \right) \left(\max_{h, x, y} \left\lvert \ell(h, x, y) \right\rvert \right) \le \frac{\sqrt{2}K}{1 - \beta}
\]
where the last inequality follows from Lemma \ref{thm:absconv} and the fact that $\ell(h, x, y) \in [0, 1]$.
\end{proof}

\begin{defn}[Cluster distortion] \label{defn:clusterdistortion} For any mechanism that takes as input a dataset $D$ and defines cluster label distributions $\tbq_D$ let
\[
\psi = \E_{D \sim \CP^n}\left[\max_c \E\left[\sum_y \left\lvert \tq_D(y | c) - \hp_{y | c}(D) \right\rvert\right]\right]
\]
be the expected maximum total variation between the empirical cluster label distributions and $\tbq_D$.
\end{defn}

\begin{lem}[Boundedness of cluster distortion] \label{thm:clusterdistortion} If $n_c(D) \ge s$ with probability 1 then
\[
\psi \le 2K\tau + \frac{2\sqrt{2}K\sigma}{s}.
\]
\end{lem}
\begin{proof} Let $[z]_+ = \max\{0, z\}$ for all $z \in \bbR$. For any label $y$ and cluster $c$
\begin{align*}
q(y | c) - \hp_{y | c}(D) &= \max\left\{\tau, \min\left\{1, \hp_{y|c}(D) + z_{y, c}\right\}\right\} - \hp_{y | c}(D) \notag\\
&\le \max\left\{\tau, \hp_{y|c}(D) + z_{y, c}\right\}- \hp_{y | c}(D) \notag\\
&\le \tau + \hp_{y|c}(D) + [z_{y, c}]_+ - \hp_{y | c}(D) \notag\\
&= \tau + [z_{y, c}]_+
\end{align*}
and
\begin{align*}
\hp_{y | c}(D) - q(y | c)  &= \hp_{y | c}(D) - \max\left\{\tau, \min\left\{1, \hp_{y|c}(D) + z_{y, c}\right\}\right\}\\
&\le \hp_{y|c}(D) - \min\left\{1, \hp_{y|c}(D) + z_{y, c}\right\}\\
&= \max\left\{\hp_{y | c}(D) - 1, -z_{y, c}\right\}\\
&\le [-z_{y, c}]_+
\end{align*}
which implies
\begin{equation}
|q(y | c) - \hp_{y | c}(D)| = \max\left\{q(y | c) - \hp_{y | c}(D), \hp_{y | c}(D) - q(y | c)\right\} \le \tau + [z_{y, c}]_+ + [-z_{y, c}]_+ = \tau + |z_{y, c}| \label{eq:one}
\end{equation}
We also have
\[
-\Delta_c = \sum_y q(y | c) - 1 \le \sum_y \left(\hp_{y | c}(D) + \tau + [z_{y, c}]_+\right) - 1 = K\tau + \sum_y [z_{y, c}]_+
\]
and
\[
\Delta_c = 1 - \sum_y q(y | c) \le 1 - \sum_y \left(\hp_{y | c}(D) - [-z_{y, c}]_+\right) = \sum_y [-z_{y, c}]_+
\]
which implies
\begin{equation}
|\Delta_c| = \max\left\{-\Delta_c, \Delta_c\right\} \le  K\tau + \sum_y [z_{y, c}]_+ + \sum_y [-z_{y, c}]_+ = K\tau + \sum_y |z_{y, c}| \label{eq:two}
\end{equation}
Therefore
\begin{align}
\left\lvert \tq(y | c) - \hp_{y | c}(D) \right\rvert &= \left\lvert q(y | c) + \frac{\xi_{y, c}}{\sum_{y'} \xi_{y', c}} \Delta_c - \hp_{y | c}(D) \right\rvert \notag\\
&\le \left\lvert q(y | c) - \hp_{y | c}(D) \right\rvert + \frac{\xi_{y, c}}{\sum_{y'} \xi_{y', c}} \left\lvert \Delta_c \right\rvert \notag\\
&\le \tau + |z_{y, c}| + \frac{\xi_{y, c}}{\sum_{y'} \xi_{y', c}} \left(K\tau + \sum_{y'} |z_{y', c}|\right) \label{eq:three}
\end{align}
where Eq.~\eqref{eq:three} follows from Eq.~\eqref{eq:one} and Eq.~\eqref{eq:two}. Therefore for any cluster $c$
\begin{align}
\E_{\bz}\left[\sum_y \left\lvert \tq(y | c) - \hp_{y | c}(D) \right\rvert\right] &\le K\tau + \sum_y \E_{\bz}\left[|z_{y,c}|\right] + \E_{\bz}\left[\frac{\sum_y \xi_{y, c}}{\sum_{y'} \xi_{y', c}} \left(K\tau + \sum_{y'} |z_{y', c}|\right) \right] \notag\\
 &= 2K\tau + 2\sum_y \E_{\bz}\left[|z_{y, c}| \right] \label{eq:four}
\end{align}
Recall that each $z_{y, c}$ has mean zero and standard deviation $\frac{\sqrt{2}\sigma}{n_c(D)}$. Continuing from Eq.~\eqref{eq:four} we have
\[
2K\tau + 2\sum_y \E_{\bz}\left[|z_{y, c}| \right] \le 2K\tau + 2 \sum_y \sqrt{\E_{\bz}\left[z^2_{y, c}\right]} = 2K\tau + \frac{2\sqrt{2}K\sigma}{n_c(D)} \le 2K\tau + \frac{2\sqrt{2}K\sigma}{s}
\]
where we used Jensen's inequality and $n_c(D) \ge s$.\end{proof}

\begin{lem}[Excess risk] \label{thm:excessrisk} If $\max_{h,x,y} |\tell(h, x, y)| \le \tL$ then for any hypothesis $h$ 
\[
\left \lvert R_{\beta, \beta}(h) - R_{\beta, \lambda}(h) \right \rvert \le \tL|\beta - \lambda|\left(\phi + \psi\right)
\]
\end{lem}
\begin{proof} Let $(\bx, \by, \rby, \tby, \bz) \sim \CP^n_{\beta, \beta}$ and $(\bx', \by', \rby', \tby', \bz') \sim \CP^n_{\beta, \lambda}$. Note that between each corresponding pair of variables only $\tby$ and $\tby'$ can have different distributions. Therefore
\begin{align}
R_{\beta, \beta}(h) - R_{\beta, \lambda}(h)&= \E_{\bx, \by, \rby, \tby, \bz}\left[\hR_{\beta, \beta}(h)\right] - \E_{\bx, \by, \rby, \tby', \bz}\left[\hR_{\beta, \lambda}(h)\right] \notag\\
&= \E_{\bx, \by, \rby, \tby, \bz}\left[\frac1n \sum_i \tell(h, x_i, \ty_i)\right] - \E_{\bx, \by, \rby, \tby', \bz}\left[\frac1n \sum_i \tell(h, x_i, \ty'_i)\right] \notag\\
&= \E_{\bx, \by, \rby, \bz}\left[\frac1n \sum_i \sum_y \left((1 - \beta) \indic{y_i = y} + \beta \indic{\ry_i = y}\right)\tell(h, x_i, y)\right] \notag\\
&~~~~ - \E_{\bx, \by, \rby, \bz}\left[\frac1n \sum_i \sum_y \left((1 - \lambda) \indic{y_i = y} + \lambda \indic{\ry_i = y}\right) \tell(h, x_i, y)\right] \notag\\
&= \E_{\bx, \by, \rby, \bz}\left[\frac1n \sum_i \sum_y \left((\lambda - \beta) \indic{y_i = y} + (\beta - \lambda) \indic{\ry_i = y}\right) \tell(h, x_i, y)\right] \notag\\
&= (\beta - \lambda) \frac1n \sum_i \E_{\bx, \by, \bz}\left[\sum_y \left(p_{y | c_{x_i}}(D) - \indic{y_i = y}\right) \tell(h, x_i, y)\right] \label{eq:firstsum}\\
&~~~~ + (\beta - \lambda) \frac1n \sum_i \E_{\bx, \by, \rby, \bz}\left[\sum_y \left(\indic{\ry_i = y} - p_{y | c_{x_i}}(D)\right)\tell(h, x_i, y)\right] \label{eq:secondsum}
\end{align}
Each term in Eq.~\eqref{eq:firstsum} is
\begin{align}
& ~~~ \E_{\bx, \by, \bz}\left[\sum_y \left(p_{y | c_{x_i}}(D) - \indic{y_i = y}\right) \tell(h, x_i, y)\right] \notag\\
& = \E_{\bx}\left[\sum_y \left(\frac{p(y | x_i)}{n_{c_{x_i}}(D)} + \frac{(n_{c_{x_i}}(D) - 1)p(y | c_{x_i})}{n_{c_{x_i}}(D)}- p(y | x_i)\right) \E_{\bz}\left[\tell(h, x_i, y)\right]\right] \notag\\
& \le \E_{\bx}\left[\sum_y \left\lvert p(y | c_{x_i}) - p(y | x_i)\right\rvert \left\lvert \E_{\bz}\left[\tell(h, x_i, y)\right] \right\rvert\right] \notag\\
& \le \tL\phi \label{eq:firsthalf}
\end{align}
where Eq.~\eqref{eq:firsthalf} follows from our assumption about $\tL$ and the definition of cluster heterogeneity in Definition \ref{defn:heterogeneity}. Each term in Eq.~\eqref{eq:secondsum} is
\begin{align}
& ~~~ \E_{\bx, \by, \rby, \bz}\left[\sum_y \left(\indic{\ry_i = y} - p_{y | c_{x_i}}(D)\right)\tell(h, x_i, y)\right] \notag\\
&= \E_{\bx, \by, \bz}\left[\sum_y \left(\tq(y | c_{x_i}) - p_{y | c_{x_i}}(D) \right) \tell(h, x_i, y)\right] \notag\\
&\le \E_{\bx, \by}\left[\E_{\bz}\left[\sum_y \left \lvert \tq(y | c_{x_i}) - p_{y | c_{x_i}}(D) \right \rvert \left \lvert\tell(h, x_i, y)\right \rvert\right]\right] \notag\\
&\le \tL \psi \label{eq:secondhalf}
\end{align}
where Eq.~\eqref{eq:secondhalf} follows from our assumption about $\tL$ and the definition of cluster distortion in Definition \ref{defn:clusterdistortion}. Combining Eq.~\eqref{eq:firstsum}, \eqref{eq:secondsum}, \eqref{eq:firsthalf} and \eqref{eq:secondhalf} proves the lemma.\end{proof}

\begin{lem}[Complexity bound] \label{thm:complexity} There exists a universal constant $C > 0$ such that
\[
\max_{h \in \CH} \left\lvert \E_{\bz}\left[\hR_{\beta, \lambda}(h)\right] - R_{\beta, \lambda}(h)\right\rvert \le \frac{CK}{1 - \beta}\sqrt{\frac{\dim(\CH, \ell)\log \frac1\gamma}{n}}
\]
with probability $1 - \gamma$. \end{lem}
\begin{proof} We first review some results from statistical learning theory \cite{mohri2018foundations}. Let $\ba = (a_1, \ldots, a_n) \in \CA^n$ be a vector of independent random variables, and let $\CF$ be a class of real-valued functions with domain $\CA^n$. We say $\CF$ has $b$-bounded differences if $|f(\ba) - f((\ba_{-i}, a_i))| \le \frac{b}{n}$ for all $f \in \CF$ and $a_i \in \CA$. If $\CF$ has $b$-bounded differences then with probability $1 - \gamma$
\begin{equation}
\max_{f \in \CF} \left|f(\ba) - \E_{\ba}\left[f(\ba)\right]\right| \le 2\mathfrak{R}(\CF, \ba) + \sqrt{\frac{b \log \frac1\gamma}{n}} \label{eq:rad1}
\end{equation}
where $\mathfrak{R}(\CF, \ba)$ is the Rademacher complexity of $\CF$ for random variable $\ba$. For any $b \ge 0$ let 
\begin{equation}
\textrm{absconv}_b(\CF) = \left\{\sum_{i=1}^N w_i f_i : N \in \bbN, \sum_{i=1}^N |w_i| \le b, f_i \in \CF\right\} \label{eq:rad2}
\end{equation}
be the absolute convex hull of $\CF$ scaled by $b$. We have
\begin{equation}
\mathfrak{R}(\textrm{absconv}_b(\CF), \ba) = b \cdot \mathfrak{R}(\CF, \ba) \label{eq:rad3}
\end{equation}
Finally, if $\CF = \{(x, y) \mapsto \ell(h, x, y)\}$ then
\begin{equation}
\mathfrak{R}(\CF, \ba) \le C\sqrt{\frac{\dim(\CH, \ell)}{n}} \label{eq:rad4}
\end{equation}
for a universal constant $C > 0$.

We now proceed to prove the lemma. We have
\begin{align}
\left\lvert \E_{\bz}\left[\hR_{\beta, \lambda}(h)\right] - R_{\beta, \lambda}(h)\right\rvert &= \left\lvert\E_{\bz}\left[\hR_{\beta, \lambda}(h)\right] - \E_{\bx, \by, \rby, \tby, \bz}[\hR_{\beta, \lambda}(h)]\right\rvert \notag\\
&\le \left\lvert \E_{\bz}\left[\hR_{\beta, \lambda}(h)\right] - \E_{\rby, \tby, \bz}[\hR_{\beta, \lambda}(h) ~|~ \bx, \by]\right\rvert \notag\\
& ~~~~ + \left\lvert \E_{\rby, \tby, \bz}[\hR_{\beta, \lambda}(h) ~|~ \bx, \by] - \E_{\bx, \by, \rby, \tby, \bz}[\hR_{\beta, \lambda}(h)]\right\rvert \notag\\
&= \left\lvert \E_{\bz}\left[\hR_{\beta, \lambda}(h)\right] - \E_{\rby, \tby, \bz}[\hR_{\beta, \lambda}(h) ~|~ \bx, \by]\right\rvert \notag\\
& ~~~~ + \left\lvert \E_{\rby, \tby, \bz}[\hR_{\beta, \lambda}(h) ~|~ \bx, \by] - \E_{\bx, \by}\left[\E_{\rby, \tby, \bz}[\hR_{\beta, \lambda}(h) ~|~ \bx, \by]\right]\right\rvert \label{eq:rad5}
\end{align}
which follows from definitions. Let $\CF', \CF''$ be the function classes
\begin{align*}
\CF' &= \left\{(\rby, \tby) \mapsto \E_{\bz}\left[\hR_{\beta, \lambda}(h)\right]\right\}\\
\CF'' &= \left\{(\bx, \by) \mapsto \E_{\rby, \tby, \bz}\left[\hR_{\beta, \lambda}(h) ~|~ \bx, \by\right]\right\}
\end{align*}
Recalling that $(\bx, \by, \rby, \tby, \bz) \sim \CP^n_{\beta, \lambda}$, note that each $(x_i, y_i)$ is independent and each $(\ry_i, \ty_i)$ is independent given $(\bx, \by, \bz)$. Continuing from Eq.~\eqref{eq:rad5}, we have with probability $1 - \gamma$
\begin{align}
\left\lvert \E_{\bz}\left[\hR_{\beta, \lambda}(h)\right] - R_{\beta, \lambda}(h)\right\rvert &\le 2\mathfrak{R}(\CF', (\rby, \tby)) + 2\mathfrak{R}(\CF'', (\bx, \by)) + 2\sqrt{\max_{h, x, y} |\tell(h, x, y)|} \sqrt{\frac{\log \frac1\gamma}{n}} \label{eq:rad6}\\
&\le 2\mathfrak{R}(\CF', (\rby, \tby)) + 2\mathfrak{R}(\CF'', (\bx, \by)) + \sqrt{\frac{2K}{1 - \beta}} \sqrt{\frac{\log \frac1\gamma}{n}} \label{eq:rad7}\\
&\le \frac{8CK}{1 - \beta}\sqrt{\frac{\dim(\CH, \ell)}{n}} + \sqrt{\frac{2K}{1 - \beta}} \sqrt{\frac{\log \frac1\gamma}{n}} \label{eq:rad8}
\end{align}
where Eq.~\eqref{eq:rad6} follows from Eq.~\eqref{eq:rad1}, Eq.~\eqref{eq:rad7} follows from Lemma \ref{thm:boundedness}, and Eq.~\eqref{eq:rad8} follows from the definition of $\tell$ in Algorithm \ref{alg:centralcurator} (see line \ref{line:modifiedlossfunction}), Lemma \ref{thm:absconv}, Eq.~\eqref{eq:rad4} and Eq.~\eqref{eq:rad4}. Combining terms proves the lemma.
\end{proof}

We are now ready to prove Theorem \ref{thm:centralizedutility}.

\begin{proof}[Proof of Theorem \ref{thm:centralizedutility}] Combining Lemma \ref{thm:boundedness}, Lemma \ref{thm:clusterdistortion}, Lemma \ref{thm:excessrisk} and our assumption that $n_c(D) \ge s$ with probability 1 we have
\begin{equation}
\max_{h \in \CH} \left \lvert R_{\beta, \beta}(h) - R_{\beta, \lambda}(h) \right \rvert \le \frac{4K}{1 - \beta}\left(\phi + K\tau + \frac{K\sigma}{s}\right) \label{eq:main1}
\end{equation}
Therefore
\begin{align}
R(\tih) - R(h^*) &= R_{\beta, \beta}(\tih) - R_{\beta, \beta}(h^*) \notag\\
&= \hR_{\beta, \lambda}(\tih) - \hR_{\beta, \lambda}(h^*) + (R_{\beta, \beta}(\tih) - R_{\beta, \lambda}(\tih)) + (R_{\beta, \lambda}(h^*) - R_{\beta, \beta}(h^*)) \notag\\
&~~~~ + (R_{\beta, \lambda}(\tih) - \hR_{\beta, \lambda}(\tih)) + (\hR_{\beta, \lambda}(h^*) - R_{\beta, \lambda}(h^*)) \notag\\
&\le 0 + \frac{8K}{1 - \beta}\left(\phi + K\tau + \frac{K\sigma}{s}\right) \notag\\
&~~~~ + (R_{\beta, \lambda}(\tih) - \hR_{\beta, \lambda}(\tih)) + (\hR_{\beta, \lambda}(h^*) - R_{\beta, \lambda}(h^*)) \label{eq:main2}
\end{align}
where Eq.~\eqref{eq:main2} follows from the choice of $\tih$ and Eq.~\eqref{eq:main1}. Taking the expectation of both sides over $\bz$ and continuing from Eq.~\eqref{eq:main2} we have
\begin{align}
\E_{\bz}[R(\tih)] - R(h^*) &\le \frac{8K}{1 - \beta}\left(\phi + K\tau + \frac{K\sigma}{s}\right) + 2\max_{h \in \CH} \left\lvert \E_{\bz}\left[\hR_{\beta, \lambda}(h)\right] - R_{\beta, \lambda}(h) \right\rvert \notag\\
&\le \frac{8K}{1 - \beta}\left(\phi + K\tau + \frac{K\sigma}{s}\right) + 2\max_{h \in \CH} \left\lvert \E_{\bz}\left[\hR_{\beta, \lambda}(h)\right] - R_{\beta, \lambda}(h) \right\rvert \label{eq:main3}
\end{align}
where Eq.~\eqref{eq:main3} follows from Lemma \ref{thm:complexity}.
\end{proof}

\section{Analysis of peer-to-peer mechanism}

\subsection{Proof of Theorem \ref{thm:peertopeerprivacy}}

First we need a technical lemma.

\begin{lem} \label{thm:approx} $\left(1 + \frac{s}{x}\right)^{s x^a} \le e^{2a- 1} + \frac{3}{x^a}$ for all $x \ge 2, s \in \{-1, 1\}, a \in \{\frac12, 1\}$\end{lem}

We next state and prove a more general version of Theorem \ref{thm:peertopeerprivacy}.

\begin{lem}[Peer-to-peer privacy, general version] \label{thm:peertopeerprivacylem} If non-empty clusters in $D$ have size at least $\frac{2}{\alpha}$ then the peer-to-peer mechanism (Algorithm \ref{alg:peertopeeruser}) satisfies $(\eps, \delta)$-label differential privacy with
%\vspace{-0.2em}
\begin{align*}
\eps &= \theta \log\left(e + \frac{3}{s\alpha}\right) + \sqrt{\theta} \xi \log\left(1 + \frac{3}{\sqrt{s\alpha}}\right)\\
\delta &= \exp\left(-\frac{\alpha\xi^2}{4\left(\alpha + \frac1s\right)(1 - \alpha)}\right)
\end{align*}
%\vspace{-0.2em}
for all $\xi \in \left[0, 3\alpha \sqrt{\theta s(1 - \alpha)}\right]$.\end{lem}
\begin{proof} Consider two neighboring datasets $D$ and $D'$ such that there is an example with label $0$ in $D$ but label $1$ in $D'$, and let $c$ be the cluster containing this example. Let $\tD$ and $\tD'$ be the output of the peer-to-peer mechanism when given $D$ and $D'$, respectively, as input. Since the labels in $\tD$ and $\tD'$ are chosen independently per cluster, the label distribution in all clusters other than $c$ is identical in both $\tD$ and $\tD'$.

Let $n = n_c(D) = n_c(D') \ge s$ be the number of examples in cluster $c$, and let $p$ be the fraction of examples in cluster $c$ with a positive label in $D$. Also let $t = \theta n$ be the fraction of users in cluster $c$ who send an example to the learner. Observe that the label distribution in cluster $c$ in $\tD$ is completely characterized by the binomial density function $f(k; t, p)$, which gives the probability of $k$ successes in $t$ trials that each have success probability $p$. Similarly, the label distribution in cluster $c$ in $\tD'$ is completely charaterized by $f(k; t, p')$, where $p' = p + \frac1n$.

Let $q = 1 - p$. Also let $S_- = \{k \in \bbN : k \ge tp - \xi\sqrt{tq}\}$ and $S_+ = \{k \in \bbN : k \le tp + \xi\sqrt{tp}\}$. Thus to prove the theorem it suffices to show
\begin{equation}
\frac{f(k; t, p)}{f(k; t, p')} \le e^\eps\textrm{ if }k \in S_- \textrm{ and }\frac{f(k; t, p')}{f(k; t, p)} \le e^\eps\textrm{ if }k \in S_+ \label{eq:privacy1}
\end{equation}
and
\begin{equation}
\sum_{k \not\in S_-} f(k; t, p) \le \delta\textrm{ and }\sum_{k \not\in S_+} f(k; t, p') \le \delta. \label{eq:privacy2}
\end{equation}
To prove the first part of Eq.~\eqref{eq:privacy1} we can simplify
\begin{equation}
\frac{f(k; t, p)}{f(k; t, p')} = \frac{p^k(1 - p)^{t - k}}{(p')^k(1 - p')^{t - k}} = \left(\frac{np}{np + 1}\right)^k\left(\frac{nq}{nq - 1}\right)^{t - k} = \left(1 + \frac{1}{np}\right)^{-k} \left(1 - \frac{1}{nq}\right)^{k - t},
\end{equation}
and if $k \in S_-$ then $k \ge tp - \xi\sqrt{tq}$ which implies
\begin{equation}
\left(1 + \frac{1}{np}\right)^{-k} \left(1 - \frac{1}{nq}\right)^{k - t} \le \left(1 - \frac{1}{nq}\right)^{-tq - \xi\sqrt{tq}} = \left(\left(1 - \frac{1}{nq}\right)^{-nq}\right)^\theta \left(\left(1 - \frac{1}{nq}\right)^{-\sqrt{nq}}\right)^{\sqrt{\theta}\xi},
\end{equation}
and by applying Lemma \ref{thm:approx} and $nq \ge s\alpha \ge 2$ we have
\begin{align}
\left(\left(1 - \frac{1}{nq}\right)^{-nq}\right)^\theta \left(\left(1 - \frac{1}{nq}\right)^{-\sqrt{nq}}\right)^{\sqrt{\theta} \xi} &\le \left(e + \frac{3}{nq}\right)^\theta \left(1 + \frac{3}{\sqrt{nq}}\right)^{\sqrt{\theta} \xi}\\ &\le \left(e + \frac{3}{s\alpha}\right)^\theta \left(1 + \frac{3}{\sqrt{s\alpha}}\right)^{\xi \sqrt{\theta}} = e^\eps.
\end{align}
Similarly, to prove the second part of Eq.~\eqref{eq:privacy1} we can simplify
\begin{equation}
\frac{f(k; t, p')}{f(k; t, p)} = \frac{(p')^k(1 - p')^{t - k}}{p^k(1 - p)^{t - k}} = \left(\frac{np + 1}{np}\right)^k\left(\frac{nq - 1}{nq}\right)^{t - k} = \left(1 + \frac{1}{np}\right)^k \left(1 - \frac{1}{nq}\right)^{t - k},
\end{equation}
and if $k \in S_+$ then $k \le tp + \xi\sqrt{tp}$ which implies
\begin{equation}
\left(1 + \frac{1}{np}\right)^k \left(1 - \frac{1}{nq}\right)^{t - k} \le \left(1 + \frac{1}{np}\right)^{tp + \xi\sqrt{tp}} = \left(\left(1 + \frac{1}{np}\right)^{np}\right)^\theta \left(\left(1 + \frac{1}{np}\right)^{\sqrt{np}}\right)^{\sqrt{\theta} \xi},
\end{equation}
and by applying Lemma \ref{thm:approx} and $np \ge s\alpha \ge 2$ we have
\begin{align}
\left(\left(1 + \frac{1}{np}\right)^{np}\right)^\theta \left(\left(1 + \frac{1}{np}\right)^{\sqrt{np}}\right)^{\sqrt{\theta} \xi} &\le \left(e + \frac{3}{np}\right)^\theta \left(1 + \frac{3}{\sqrt{np}}\right)^{\sqrt{\theta} \xi}\\ &\le \left(e + \frac{3}{s\alpha}\right)^\theta \left(1 + \frac{3}{\sqrt{s\alpha}}\right)^{\sqrt{\theta} \xi} = e^\eps.
\end{align}

To prove the first part of Eq.~\eqref{eq:privacy2} define the binomial cumulative distribution function $F(k; t, p) = \sum_{k' \le k} f(k; t, p)$. By Bernstein's inequality 
\[
F(tp - t\gamma; t, p) \le \exp\left(-\frac{\gamma^2t}{2pq + 2\gamma/3}\right)
\]
for all $\gamma > 0$. Let $\gamma = \xi\sqrt{\frac q t}$ and note that
\[
\frac23 \gamma = \frac23 \xi \sqrt{\frac q t} \le 2 \alpha \sqrt{\theta s(1 - \alpha)} \sqrt{\frac{q}{t}} =  2 \alpha \sqrt{1 - \alpha}\sqrt{q} \sqrt{\frac{s}{n}} \le 2 (1 - q)q \sqrt{\frac{s}{n}} \le 2pq,
\]
and thus
\[
\sum_{k \in S_-} f(k; t, p) = F(tp - t\gamma; t, p) \le \exp\left(-\frac{\gamma^2t}{4pq}\right) \le \exp\left(-\frac{\xi^2}{4p}\right) \le \exp\left(-\frac{\xi^2}{4(1 - \alpha)}\right) \le \delta.
\]
To prove the second part of Eq.~\eqref{eq:privacy2} let $\gamma = \xi \sqrt{\frac p t}$ and $q' = 1 - p'$. We have
\[
\frac23 \gamma = \frac23 \xi \sqrt{\frac p t} \le 2 \alpha \sqrt{\theta s(1 - \alpha)} \sqrt{\frac{p}{t}} =  2 \alpha \sqrt{1 - \alpha}\sqrt{p} \sqrt{\frac{s}{n}} \le 2 (1 - p)p \sqrt{\frac{s}{n}} \le 2 (1 - p')p' = 2p'q',
\]
and since $1 - F(tp' + t\gamma; t, p')  = F(tq' - t\gamma; t; q')$ we have
\[
\sum_{k \in S_+} f(k; t, p') = F(tq' - t\gamma; t; q') \le \exp\left(-\frac{\gamma^2t}{4p'q'}\right) \le \exp\left(-\frac{\alpha\xi^2}{4\left(\alpha + \frac1s\right)(1 - \alpha)}\right) = \delta. \qedhere
\]
\end{proof}

We are now ready to prove the Theorem \ref{thm:peertopeerprivacy}.

\begin{proof}[Proof of Theorem \ref{thm:peertopeerprivacy}] Let $\xi = 4\sqrt{\log s}$. Since $\alpha = \frac{4\sqrt{2} \log s}{\sqrt{\theta s}} \le \frac12$
\[
\xi = 4\sqrt{\log s} \le 4\log s \le \alpha \sqrt{\frac{\theta s}{2}} \le \alpha \sqrt{\theta s (1 - \alpha)}
\]
where the last inequality uses $\alpha \le \frac12$. Therefore the conditions of Lemma \ref{thm:peertopeerprivacylem} hold. Also
\[
\frac{\alpha}{(\alpha + \frac1s)(1 - \alpha)} \ge \frac{\alpha}{\alpha + \frac1s} = \frac{1}{1 + \frac{1}{s\alpha}} \ge \frac12
\]
since $s \ge \frac1\alpha$. Therefore by Lemma \ref{thm:peertopeerprivacylem}
\[
\delta = \exp\left(-\frac{\alpha\xi^2}{4\left(\alpha + \frac1s\right)(1 - \alpha)}\right) \le \exp\left(-\frac{\xi^2}{8}\right) = e^{-2\log s} = \frac{1}{s^2},
\]
and since $\alpha = \frac{\xi^2}{2\sqrt{2}\sqrt{\theta s}}$ it follows from Lemma \ref{thm:peertopeerprivacylem} that
\begin{align}
\eps = \theta \log\left(e + \frac{3}{s\alpha}\right) + \sqrt{\theta} \xi \log\left(1 + \frac{3}{\sqrt{s\alpha}}\right) &= \theta \log\left(e + \frac{\sqrt{72\theta}}{\sqrt{s} \log s}\right) + \sqrt{\theta} \xi \log\left(1 + \frac{2^{3/4}3\theta^{1/4}}{s^{1/4}\xi}\right)\\
&\le \theta + \frac{3 \theta^{3/2}}{e \sqrt{s} \log s} + \frac{2^{3/4}3\theta^{3/4}}{s^{1/4}} \qedhere
\end{align}
\end{proof}

\section{Hardness result}

\subsection{Proof of Theorem \ref{thm:hard}}

\begin{proof}
Fix a set of examples $\CX$. Select a pair of neighboring datasets $D \in (\CX \times \CY)^n$ and $D' \in (\CX \times \CY')^n$ as follows. 

To construct $D$, for each cluster $c\in \CC$, select $s$ labels uniformly at random without replacement and assign them to the examples in $c$. Examples in different clusters may have the similar labels. To construct $D'$ from $D$, select an example $x_i$ uniformly at random form $\CX$ and redraw its label uniformly at random from $\CY\setminus \CY_{c_{x_i}}$.
We use $i$ to denote the index of the data that differs between $D$ and $D'$, with the datapoints being $(x_i,y_i)$ and $(x_i,y'_i)$ respectively. 
Let mechanism $M$ be an $\eps$-differential privacy mechanism for LAP that guarantees a $\phi$ precision and an $\eta$ recall. Denote $\tD=M(D)$ and $\tD'=M(D')$. 

Note that by construction of $\tD'$, for each cluster $c$ we have $|\CY'_c|=s$. Hence, we have 
\begin{align*}
\eta&\leq 
\frac{\sum_{c \in \CC}\sum_{y'\in\CY'_c} \E_{M}[\tD'(c,y')] }
{\sum_{c \in \CC}|\CY'_c|} \\
&= \frac{\sum_{(x,y')\in D'}\E_M[\tD'(c_x,y')]}{n}.
\end{align*}
This means that for a [uniformly] random $(x,y')\in D'$ we have $\tD'(c_x,y')=1$ with probability at least $\eta$. Recall that index $i$ that indicates the difference of $D$ and $D'$ is chosen uniformly at random from $\{1,\dots,n\}$. Let $O$ be the set of all possible models that can be generated by $M(.)$ that contains $(c_{x_i},y'_i)$. By definition of differential privacy we have $\Pr[\tD'\in O]\leq e^{\eps} \Pr[\tD \in O]$. This implies that
\begin{align*}
    \Pr[\tD \in O] \geq e^{-\eps} \Pr[\tD'\in O] \geq e^{-\eps} \eta.
\end{align*}
Hence, with probability at least $e^{-\eps}\eta$, we have $\tD(c_{x_i},y'_i)=1$. Recall that by construction $y'_i$ is a label chosen uniformly at random from $\CY\setminus \CY_{c_{x_i}}$. Hence each any cluster $c$ is associated with any label $y'_i\notin \CY_c$ with probability $e^{-\eps}\eta$. Therefore, we have $$\E[|\tD|] \geq \frac n s\times(K-s)\times e^{-\eps}\eta.$$ Hence, we can bound the precision of $\tD$ by
\begin{align*}
    \phi &\leq \frac{\sum_{c \in \CC}\sum_{y\in\CY_c} \E_{M}[\tD(c,y)] } {\E_M[|\tD|]}\\
    &\leq \frac{\sum_{c \in \CC}\sum_{y\in\CY_c} \E_{M}[\tD(c,y)] }{\frac n s (K-s) e^{-\eps}\eta}\\
    &\leq \frac{n}{\frac n s(K-s) e^{-\eps}\eta} \\
    &= \frac s {(K-s) e^{-\eps}\eta}.
\end{align*}
This gives us $\phi \eta e^{-\eps}\leq \frac s {(K-s)} \in o(1)$. Hence, for a constant $\eps$, either precision $\phi$ is sub-constant or recall $\eta$ is sub-constant.
\end{proof}

\end{document}

%% file: macros.tex
\newtheorem{thm}{Theorem}
\newtheorem{lem}{Lemma}
\newtheorem{defn}{Definition}
\newtheorem{cor}{Corollary}

\newcommand{\CA}{\mathcal{A}}
\newcommand{\CC}{\mathcal{C}}
\newcommand{\CF}{\mathcal{F}}
\newcommand{\CH}{\mathcal{H}}
\newcommand{\CO}{\mathcal{O}}
\newcommand{\CP}{\mathcal{P}}
\newcommand{\CX}{\mathcal{X}}
\newcommand{\CY}{\mathcal{Y}}

\newcommand{\bone}{\mathbf{1}}
\newcommand{\ba}{\mathbf{a}}

\newcommand{\bq}{\mathbf{q}}
\newcommand{\bu}{\mathbf{u}}
\newcommand{\bv}{\mathbf{v}}
\newcommand{\bx}{\mathbf{x}}
\newcommand{\by}{\mathbf{y}}
\newcommand{\bz}{\mathbf{z}}
\newcommand{\bI}{\mathbf{I}}
\newcommand{\bM}{\mathbf{M}}

\newcommand{\bQ}{\mathbf{Q}}

\newcommand{\tih}{\tilde{h}}
\newcommand{\tell}{\tilde{\ell}}

\newcommand{\tq}{\tilde{q}}
\newcommand{\ty}{\tilde{y}}
\newcommand{\tD}{\tilde{D}}
\newcommand{\tL}{\tilde{L}}
\newcommand{\tO}{\tilde{O}}
\newcommand{\tQ}{\tilde{Q}}

\newcommand{\tbQ}{\tilde{\bQ}}
\newcommand{\tbq}{\tilde{\bq}}

\newcommand{\hp}{\hat{p}}

\newcommand{\hR}{\hat{R}}

\newcommand{\ry}{\mathring{y}}
\newcommand{\rby}{\mathring{\by}}
\newcommand{\tby}{\tilde{\by}}

\newcommand{\bbN}{\mathbb{N}}
\newcommand{\bbR}{\mathbb{R}}

\newcommand{\goes}{\rightarrow}
\DeclareMathOperator{\E}{E}
\newcommand{\eps}{\epsilon}
\providecommand{\textPr}[1]{\textstyle\Pr_{#1}\displaystyle}
\newcommand{\indic}[1]{\bone\left\{#1\right\}}
\providecommand{\norm}[1]{\left\lVert#1\right\rVert}